\documentclass[journal]{IEEEtran}
\ifCLASSINFOpdf
\else
\fi
\usepackage{amsmath}
\usepackage{helvet} 
\usepackage{courier}  
\usepackage[hyphens]{url}  
\usepackage{times} 
\usepackage{graphicx}
\usepackage{amsfonts}

\usepackage{booktabs}
\usepackage{caption}
\newcommand{\tabincell}[2]{\begin{tabular}{@{}#1@{}}#2\end{tabular}}

\usepackage{amsthm}
\usepackage{cases}
\usepackage{bm}
\usepackage{algorithm}
\usepackage{algorithmic}
\usepackage{cases}
\usepackage{subfigure}
\usepackage{multirow}
\usepackage{color} 
\usepackage{enumerate}
\usepackage{etoolbox}
\usepackage{mathrsfs}
\newcommand{\MyMapTemplatePrefix}[4]{\expandafter#1\csname#3#4\endcsname{#2{#4}}}
\newcommand{\MyMapTemplatePrefixNew}[5]{\expandafter#1\csname#4#5\endcsname{#2{#3{#5}}}}
\forcsvlist{\MyMapTemplatePrefix {\def} {\mathbf} {}} {A,B,C,D,E,F,G,H,I,J,K,L,M,N,O,P,Q,R,S,T,U,V,W,X,Y,Z}
\forcsvlist{\MyMapTemplatePrefix {\def} {\mathbf} {}} {a,b,c,d,e,f,g,h,i,j,k,l,m,n,o,p,q,r,s,t,u,v,w,x,y,z,1,0}
\forcsvlist{\MyMapTemplatePrefix {\def} {\widetilde} {wt}} {A,B,C,D,E,F,G,H,I,J,K,L,M,N,O,P,Q,R,S,T,U,V,W,X,Y,Z}
\forcsvlist{\MyMapTemplatePrefix {\def} {\widetilde} {wt}} {a,b,c,d,e,f,g,h,i,j,k,l,m,n,o,p,q,r,s,t,u,v,w,x,y,z} 
\forcsvlist{\MyMapTemplatePrefixNew {\def} {\widetilde}{\mathbf} {tb}} {A,B,C,D,E,F,G,H,I,J,K,L,M,N,O,P,Q,R,S,T,U,V,W,X,Y,Z}
\forcsvlist{\MyMapTemplatePrefixNew {\def} {\widetilde}{\mathbf} {tb}} {a,b,c,d,e,f,g,h,i,j,k,l,m,n,o,p,q,r,s,t,u,v,w,x,y,z}
\forcsvlist{\MyMapTemplatePrefix {\def} {\widehat} {wh}} {A,B,C,D,E,F,G,H,I,J,K,L,M,N,O,P,Q,R,S,T,U,V,W,X,Y,Z}
\forcsvlist{\MyMapTemplatePrefix {\def} {\widehat} {wh}} {a,b,c,d,e,f,g,h,i,j,k,l,m,n,o,p,q,r,s,t,u,v,w,x,y,z}
\forcsvlist{\MyMapTemplatePrefixNew {\def} {\widehat}{\mathbf} {hb}} {A,B,C,D,E,F,G,H,I,J,K,L,M,N,O,P,Q,R,S,T,U,V,W,X,Y,Z}
\forcsvlist{\MyMapTemplatePrefixNew {\def} {\widehat}{\mathbf} {hb}} {a,b,c,d,e,f,g,h,i,j,k,l,m,n,o,p,q,r,s,t,u,v,w,x,y,z}
\forcsvlist{\MyMapTemplatePrefixNew {\def} {\overline}{\mathbf} {lb}} {A,B,C,D,E,F,G,H,I,J,K,L,M,N,O,P,Q,R,S,T,U,V,W,X,Y,Z}
\forcsvlist{\MyMapTemplatePrefixNew {\def} {\overline}{\mathbf} {lb}} {a,b,c,d,e,f,g,h,i,j,k,l,m,n,o,p,q,r,s,t,u,v,w,x,y,z}
\forcsvlist{\MyMapTemplatePrefix {\def} {\mathcal}{mc}} {A,B,C,D,E,F,G,H,I,J,K,L,M,N,O,P,Q,R,S,T,U,V,W,X,Y,Z}
\forcsvlist{\MyMapTemplatePrefix {\def} {\mathbb} {mb}} {A,B,C,D,E,F,G,H,I,J,K,L,M,N,O,P,Q,R,S,T,U,V,W,X,Y,Z}

\newtheorem{prop}{Proposition}

\usepackage[numbers]{natbib}
\begin{document}
%
\title{Spatial-Temporal Tensor Graph Convolutional Network for Traffic Prediction}
%
%
%

\author{Xuran~Xu*,
        Tong~Zhang*,
        Chunyan~Xu,
        Zhen~Cui,
        and~Jian~Yang
        \thanks{Xuran Xu, Tong Zhang, Zhen Cui, Chunyan Xu and Jian Yang are with the School of Computer Science and Engineering, Nanjing University of Science and Technology, China.  \protect\\
Email address: (xuxuran,  tong.zhang, zhen.cui, cyx, csjyang) @njust.edu.cn.}
       \thanks{*Xuran Xu and Tong Zhang have equal contributions.}
       \thanks{Zhen Cui is the corresponding author.}
}

%
%

\markboth{Journal of \LaTeX\ Class Files,~Vol.~14, No.~8, August~2015}%
{Shell \MakeLowercase{\textit{et al.}}: Bare Demo of IEEEtran.cls for IEEE Journals}
%

\maketitle





\begin{abstract}
Accurate traffic prediction is crucial to the guidance and management of urban traffics. However, most of the existing traffic prediction models do not consider the computational burden and memory space when they capture spatial-temporal dependence among traffic data. In this work, we propose a factorized Spatial-Temporal Tensor Graph Convolutional Network to deal with traffic speed prediction. Traffic networks are modeled and unified into a graph that integrates spatial and temporal information simultaneously. We further extend graph convolution into tensor space and propose a tensor graph convolution network to extract more discriminating features from spatial-temporal graph data. To reduce the computational burden, we take Tucker tensor decomposition and derive factorized a tensor convolution, which performs separate filtering in small-scale space, time, and feature modes. Besides, we can benefit from noise suppression of traffic data when discarding those trivial components in the process of tensor decomposition. Extensive experiments on two real-world traffic speed datasets demonstrate our method is more effective than those traditional traffic prediction methods, and meantime achieves state-of-the-art performance.
\end{abstract}


%
\IEEEpeerreviewmaketitle

\section{Introduction}
%
%
%
%
\IEEEPARstart{A}{ccurate} traffic prediction plays an important role in promoting the development of the Intelligent Transportation System (ITS). The content of traffic prediction includes the forecasting of traffic flow, traffic speed, traffic demand, traffic space occupancy, etc. Among them, traffic speed/flow is one of the most crucial to describe the traffic situation, and traffic speed/flow prediction targets at predicting future traffic speed/flow at those specified traffic locations by using historical traffic speed/flow. With the future traffic speed/flow provided, the prediction can be of great significance to traffic signal control
systems to facilitate urban traffic planning, traffic management, and traffic control.

Since traffic prediction has a wide range of applications, numerous methods have been proposed to tackle this issue. The early works mainly focused on statistical models, e.g., Historical Average~(HA), Auto-Regressive Integrated Moving Average~(ARIMA)~\cite{ahmed1979analysis}, Vector Auto-Regressive~\cite{zivot2006vector}, Hidden Markov Model~\cite{qi2014hidden}, Gaussian Process~\cite{xie2010gaussian}, etc. However, as traffic data do not always satisfy the assumption of linearity and stationarity~\cite{ye2020build}, these methods fail to obtain satisfying performance in practice. Besides these statistical models, there are plenty of machine learning methods applied to traffic prediction, such as Support Vector Machine~\cite{fu2016vehicle}, K-Nearest Neighbors~\cite{may2008vector}, Random Forest Regression~\cite{johansson2014regression}, but the ability of their feature learning is limited due to the shallow architecture. 
Recently, deep learning-based techniques are flourishing in the task of traffic prediction. For example, convolution neural network (CNN) is used to extract correlations of the spatial domain and recurrent neural network (RNN) or its variants is leveraged to learn time-series patterns in~\cite{wu2016short}. The downside CNN fails to capture the topology of traffic networks. To solve this problem, the layout of spatial locations is viewed as a structured graph, then a graph convolution network~(GCN) is employed to extract structured spatial information in~\cite{zhao2019t}. Afterwards, the attention mechanism is introduced to model more discriminating spatial-temporal dependence in~\cite{zhu2020a3t}.

Although notable success has been achieved by these aforementioned methods, there are still several challenging issues to be tackled, which include two main points. First, the complex spatial-temporal correlations inside traffic data are still not sophisticatedly modeled. Correlations may exist among spatial locations which form an irregular topology with the traffic/flow evolving. For example, congestion in one position tends to cause congestion in surrounding places. Some previous works, e.g. the spatial-temporal graph neural network method~\cite{zhao2019t}, have attempted to capture the spatial-temporal dependence by combining GCN (modeling spatial information) and Gated Recurrent Unit (GRU) (modeling temporal information). However, this hierarchical architecture may not be powerful enough to well jointly/simultaneously model the spatial-temporal dependence as it divides the spatial-temporal learning process into two separate phases. Second, redundant components/noise may exist in the collected traffic data, which may degrade the speed/flow prediction performance. However, few existing algorithms have considered alleviating the influence caused by the noise/redundancy. Third, the prediction for large traffic networks confronts with the enormous computational burden.

Considering the challenges and difficulties above, in this work, we propose a Spatial-Temporal Tensor Graph Convolutional Network (ST-TGCN) framework to deal with the traffic speed prediction task. To facilitate the simultaneous modeling of both spatial and temporal dependence, a tensor graph with integrated spatial-temporal architecture is first constructed. Then, a tensor graph convolution is proposed to infer on this tensor graph to extract more discriminative features. Based on this tensor graph learning process, the crucial spatial-temporal dependence could be well modeled, which may promote traffic speed prediction performance. Further considering the redundant components/noise in the traffic data as well as the high memory overhead and computational burden of the graph inference on tensor, we accordingly derive a factorized tensor graph convolution operation to well approximate the tensor graph convolution. Specifically, the constructed tensor graph is first factorized into three separate modes (w.r.t space, time, feature) through tensor decomposition, then the factorized tensor graph convolution is conducted through graph convolution correspondingly on each separate mode. The factorized convolution can not only benefit the ST-TGCN in noise suppression through the tensor decomposition which discards redundant components but also reduce the memory requirements and computational cost. We verify the proposed ST-TGCN framework on two public traffic speed prediction datasets, and the experimental results verify its effectiveness.

In summary, our contributions are three folds:
\begin{enumerate}
\item[(1)] We propose a novel spatial-temporal tensor graph convolution framework to simultaneously model the spatial and temporal dependence which are crucial for traffic speed prediction. To the best of our knowledge, this is the first work to construct tensor graphs on traffic data and conduct tensor graph convolution to learn high level feature representation for the traffic speed prediction task.
\item[(2)] We propose a factorized tensor graph convolution to effectively and efficiently infer on a tensor graph. The factorized convolution can not only suppress the noise by discarding redundant components, but also reduce the memory requirement and computational cost.
\item[(3)] We evaluate the proposed ST-TGCN framework on two public traffic speed prediction datasets and report the state-of-the-art performances.
\end{enumerate}

\section{Related Work}
In this section, we first introduce some algorithms about graph convolutional networks, then review the previous methods targeting traffic prediction.

\subsection{Graph Convolutional Networks}

Graph Convolutional Network (GCN) is a natural extension of Convolutional Neural Network (CNN) on graph-structured data and has shown promising performance in various applications. Previous works adopted GCN to cope with static graph data, which can be categorized into spectral-based and spatial-based ones. The former~(e.g.,~\cite{defferrard2016convolutional,bruna2013spectral,kipf2016semi}) usually establish tools in  Graph Signals Processing~(GSP)~\cite{shuman2015multiscale}, and the latter~(e.g.,~\cite{niepert2016learning}) tends to form neighbor nodes' features into regularization.

Graph-structured data is ubiquitous in the real world, such as recommender systems, social networks, computational biology, and traffic systems, where some of them possess the spatial-temporal architecture. To deal with this specific part of data, recent works propose to utilize various models to model the spatial-temporal correlations. As a representative, Nguyen et al. \cite{nguyen2018continuous} proposed a continuous-time dynamic network that is capable of learning time-respecting embeddings from continuous-time dynamic networks.
To better capture temporal dependency between graph data, some works integrate GCN with RNN.~\citet{seo2018structured} proposed graph convolutional recurrent neural networks that utilized GCN to identify spatial structures and RNN to find dynamic patterns.~\citet{pareja2020evolvegcn} utilized RNN to evolve GCN parameters to capture the dynamism of the input graph sequence. Recently, the attention mechanism has been embedded into GCN to further discover the crucial temporal patterns, e.g., \cite{zhu2020a3t} and \cite{guo2019attention} adopt attention mechanism to assign different weights to historical information. However, it is still difficult to model long-term dependency among the high dimensional data with spatial-temporal graph convolutional network.
\subsection{Traffic Prediction}
Early approaches to cope with the traffic prediction problem are generally based on data statistic techniques, among which the Historical Average (HA) and the Auto-Regressive Integrated Moving Average (ARIMA) \cite{ahmed1979analysis} are two representative models. The HA model computes an average of historical traffic speed as a prediction value. Hence, it fails to capture spatial correlation and complex temporal correlations. Although ARIMA linearly combines historical traffic data, it is merely applicable to stationary data. Then traditional machine learning algorithms are introduced to dealing with the traffic prediction problem, e.g., Support Vector Regression~(SVR)~\cite{drucker1997support,smola2004tutorial}  predict traffic speed regressively.~\citet{qi2014hidden} proposed a Hidden Markov Model to model traffic behavior as a stochastic process then predict traffic speed with state transition probabilities. For more details,~\citet{qi2010probabilistic} comprehensively summarized probabilistic models for short term traffic predictions. Moreover, K-Nearest Neighbors~\cite{may2008vector,cai2016spatiotemporal} is also applied to the short-term traffic prediction,~\citet{may2008vector} proposed a k-nearest neighbor algorithm to capture spatial correlations among different locations based on the definition of an appropriate distances function. 

Recently, deep learning-based approaches have been adopted to cope with traffic prediction problem. To capture long-term temporal dependence,~\citet{ma2015long} utilized Long Short-Term Neural Networks~(LSTM) to overcome the issue of back-propagated error decay.~\citet{zhang2016dnn} designed a deep learning-based prediction model based on spatial-temporal domain knowledge. Although both spatial and temporal dependence are considered in this model, the regular grid is inconsistent with the actual layout of traffic networks. To model non-Euclidean structure data, ~\citet{yu2017spatio} utilized spatial-temporal graph convolutional networks~(STGCN) to capture comprehensive spatial-temporal correlations with multi-scale traffic networks. Specifically, STGCN employed the fully convolutional structure on the temporal dimension to overcome the inherent deficiencies of RNN. Moreover, \citet{li2017diffusion} and~\citet{zhao2019t} respectively proposed Diffusion Convolutional Recurrent Neural Network that combines GCN with RNN and Temporal GCN~(T-GCN) that combines GCN with GRU to model spatial-temporal patterns of traffic data. Afterwards, \citet{zhu2020a3t} and ~\citet{park2019stgrat} embedded attention mechanism into T-GCN to further capture the dynamic traffic patterns.

\section{Preliminaries}
\subsection{Tensor Operation}
We will briefly introduce the standardized tensor conception according to the literature~\cite{kiers2000towards,kolda2009tensor}. To distinguish tensor, matrix, and vector, we respectively denote them with boldface Euler script, bold uppercase, and bold lowercase. e.g., $\mcX,\X$, and $\x$. Tensor is defined as a multi-dimensional array and a $M$-dimensional tensor can be represented as $\mcX{\in{\mathbb{R}^{I_1\times{I_2}\times\cdots\times{I_M}}}}$, where $M$ is the number of dimensions and $I_M$ indicates the size of the $M$-th dimension of tensor $\mcX$. Similar to the use of matrix notation, the element at the index position $(i_1,i_2,\cdots,i_M)$ is denoted as $\mcX_{i_1i_2\cdots i_M}$, and we use the colon to denote the full range of a given index. 
Furthermore, a tensor can be reordered into a matrix by unfolding or flattening, the mode-n matricization of a tensor $\mcX$ is denoted as $\X_{(n)}$, the size is $I_n\times{\big(\prod_{m=1,m\neq{n}}^{N} I_m\big)}$. The $n$-mode product of a tensor is used to calculate a multiplication with a matrix in the mode $n$. The $n$-mode product of  $\mcX\in\mathbb{R}^{I_1\times{I_2}\times\cdots\times{I_M}}$ with a matrix $\U\in{\mathbb{R}^{J\times{I_n}}}$ is denoted as $\mcX\times_n{\U}$. Formally, the tensor product is calculated as
\begin{align}
(\mcX\times_n{\U})_{i_1\cdots i_{n-1}ji_{n+1}\cdots{i_M}}=\sum_{i_n=1}^{I_n}\mcX_{i_1i_2...i_M}\U_{ji_n},
\end{align}
where $(\mcX\times_n{\U})\in
\mathbb{R}^{I_1\times\cdots\times{I_{n-1}}\times{J}\times{I_{n+1}}\times\cdots\times{I_M}}$.

In addition, the tensor-matrix multiplication satisfies the commutative law and the associative law:
\begin{itemize}
	\item[-] Given matrices $\U\in\mbR^{J \times{I_n}}$,$\V\in\mbR^{K\times{I_m}}$ and $n\neq{m}$, then we have
	\begin{align}
	\mcX\times_m{\U}\times_n{\V}&=\mcX\times_n{\V}\times_m{\U}
	\end{align}
	\item[-] Given matrices $\U\in\mbR^{J \times{I_n}}$,$\V\in\mbR^{{K}\times{J}}$ , then
	\begin{align}
	\mcX\times_n{\U}\times_n{\V}=\mcX\times_n{({\V}{\U})}.
	\end{align}
	
\end{itemize}
Tensor decomposition is often used to capture the principal components of a tensor, which is widely used for signal processing~\cite{sidiropoulos2017tensor}. As a representative, Tucker decomposition~\cite{tucker1951method,tucker1964contributions,tucker1966some} is regarded as a higher-order generalization of matrix singular value decomposition~(SVD) and principal component analysis~(PCA). It factorizes a tensor into a much smaller core tensor and factor matrices along with each mode. Concretely, the tensor $\mcX\in\mathbb{R}^{I_1\times{I_2}\times\cdots\times{I_M}}$ can be decomposed into
\begin{align}
\mcX=\mcC\times_1{\A^{(1)}}\times_2{\A^{(2)}}...\times_N{\A^{(M)}},
\end{align}
where $\mcC\in\mbR^{d_1\times d_2\times\cdots\times d_M}$ is a core tensor, $d_i$ donates the numbers of components used to summarize the entities in the mode $i$ and $\A^{(i)}\in\mbR^{ I_i \times d_i}$ is the factor matrix of the $i$-th mode ~\cite{kiers2001three}. In particular, below we consider the case that the order of $\mcX$ is equal to 3, which can be written as $\mcX{\in{\mathbb{R}^{I_1\times{I_2}\times{I_3}}}}$. But our method could be easily generalized into higher-order tensors without any extra limitations/constraints.

\subsection{Graph Convolution}
From the perspective of graph signal processing, spectral convolution on graphs was introduced in the literature~\cite{bruna2013spectral}. A graph is defined as a triple $\texttt{G}=(\texttt{V},\A,\x)$, where $\texttt{V}$ is a set of $N$ nodes, $\A\in\mbR^{N\times N}$ is the adjacency matrix, and $\x\in\mbR^{N}$ is the graph signal. Each node is associated with a signal.

Spectral convolution~\cite{chung1997spectral} is defined as $\y=\U g(\mathbf{\Lambda}) \U^\top\x$, where $g(\cdot)$ is a filter function on spectrum matrix $\Gamma$, $\U$ is a transformer from space domain to frequency domain. We may use normalized graph Laplacian $\L=\I_N-\D^{-\frac{1}{2}}\A\D^{\frac{1}{2}}$ or normalized adjacency matrix $\tbA=\D^{-\frac{1}{2}}\A\D^{\frac{1}{2}}$ to obtain $\U, \Gamma$ through singular value decomposition (SVD). Here $\I_N\in\mbR^{N\times{N}}$ is an identity matrix and $\D\in\mbR^{N\times N}$ is the degree matrix.
Formally, we have $\L=\U\mathbf{\Lambda}\U^\top$ or $\tbA=\U\mathbf{\Lambda}{\U}^\top$, where $\mathbf{\Lambda}$ is a diagonal matrix.

Due to the expensive computation of eigenvalue decomposition, we may use some polynomial to approximate filter function $g(\cdot)$ as used in ~\cite{defferrard2016convolutional}. In a simple fashion, we use the $n$-order polynomial to approximate $g(x)\approx\sum_{i=0}^n \theta_i x^i$. Accordingly, the convolution on the graph can be rewritten as
\begin{align}
\y=\sum_{k=0}^{n} \theta_k\L^k \x \quad \text{or} \quad \y=\sum_{k=0}^{n} \theta_k\tbA^k \x,
\end{align}
where ${\theta_k| k=0, \cdots, n}$ is the filter parameters to be learned. Deeper details/contents could be found in the literature in the above related work.
\section{Problem Description}
It is a natural process to view traffic networks as spatial-temporal graph-structure data. For traffic networks, we can use the adjacency matrix $\A$ to describe connection relations between different traffic locations. At each timestamp, we construct a graph $\verb"G"$ to express the states of the traffic networks, denoted as $\texttt{G}=(\texttt{V},\A,\X)$, where $\texttt{V}$ is a set of $N$ nodes w.r.t traffic locations. The state matrix $\X\in\mbR^{N\times D}$ records traffic flows/speeds of all nodes, where each row corresponds to one node and $d$ is the feature dimensionality of nodes. The adjacency matrix $\A\in\mbR^{N\times N}$ describes connections between different nodes (i.e., locations).

In the task of traffic prediction, at the time stamp $t$, given the historical observations from previous $T$ time stamps, we expect to accurately estimate the node states of the next $T'$ time slice. Formally,
\begin{align}\small
(\hby^{(t+1)},\hby^{(t+2)},\cdots,\hby^{(t+T')})=\varphi(\texttt{G}^{(t)},\texttt{G}^{(t-1)},\cdots,\texttt{G}^{(t-T+1)}),
\end{align}
where $\varphi$ is the prediction function, and $\texttt{G}^{(t)}=(\texttt{V},\A,\X^{(t)})$. As the traffic locations are usually invariant, commonly, we need to only predict the future traffic speeds/flows $(\hby^{(t+1)},\hby^{(t+2)},...,\hby^{(t+T')})$ such that the traffic condition could be coordinated in time. In the following section, we reduce this problem into the graph scheme and propose a spatial-temporal tensor graph convolution method to estimate the states of future time slices.

\section{Spatial-Temporal Tensor Graph Convolution}

Given the observed historical traffic states of the previous $T$ time slice, we can use one structured graph to model the traffic information at each timestamp. Thus we can have these graph data $\{\texttt{G}^{(t)},\texttt{G}^{(t-1)},\cdots,\texttt{G}^{(t-T+1)}\}$, our aim is to estimate future traffic speeds/flows. Considering the complicated spatial-temporal correlations between traffic networks, we respectively conduct graph convolution on spatial and temporal domains, which could be regarded as multi-graph convolutional networks.

Concretely, we collect the features/embeddings of all nodes,  $\{\X^{(t)},\X^{(t-1)},...,\X^{(t-T+1)}\}$ w.r.t previous $T$ time slices, and stack them into a 3-D tensor $\mcX\in {\mbR^{{N}\times{D}\times{T}}}$. The frontal slices $\mcX_{::t}=\mathbf{X}^{(t)}$ represents nodes' embeddings at time $t$.
Further, we formulate spatial-temporal graph convolution into the tensor product,
\begin{align}
\widetilde{\mcX}=\sum_{k_\text{S},k_\text{T}=0}^{p} \mcX\times_1 \A_{\text{S}}^{k_\text{S}}~\widetilde{\times}_3~\mcA_{\text{T}}^{k_\text{T}} \times_2\Theta_{k_\text{S}k_\text{T}},
\label{eqn:wt_mcX}
\end{align}
where
\begin{itemize}
	\item[-] Batch tensor product $\widetilde{\times}_{3}$: The operator performs batch processing in the first dimension. Given two tensors, $\mcA\in\mbR^{N\times D\times T}$ and $\mcB\in\mbR^{N\times T\times T}$, we define the calculation as follows:
	\begin{align}
	[\mcA~\widetilde{\times}_3~\mcB]_{n::} = \mcA_{n::}\times~\mcB_{n::}.\label{eqn:batch_cal}
	\end{align}
	\item[-] Spatial adjacency matrix $\A_\text{S}\in\mbR^{N\times N}$: It expresses the relations between different spatial traffic locations. In practice, due to the  invariance of traffic locations, we set the shared spatial adjacent relations, i.e., a common adjacency matrix $\A_\text{S}$ for all time slices.
	\item[-] Temporal adjacency tensor $\mcA_\text{T}\in\mbR^{N\times T\times T}$: The connections in the temporal are built into the tensor for each node.
	Together with spatial adjacent relations, we define the temporal connection structure of traffic networks.
	\item[-] Filtering parameters $\Theta\in\mbR^{D'\times D}$: The original feature space of $D$-dimension would be embedded into a new $D'$-dimension space through a tensor product with $\Theta$, which is required to be learned during training.
	\item[-] The order number $k_\text{S}, k_\text{T}\in\verb"N"$: The power calculation in the spatial domain, i.e., $\A_{\text{S}}^{k_\text{S}}$, follows the general matrix power operation. The power operation on tensor takes the way of batch processing at the first dimension, i.e., $[\mcA_{\text{T}}^{k_\text{T}}]_{n::}=[\mcA_{\text{T}}]_{n::}^{k_\text{T}}$. For the zero-order calculation, we have $\A^0=\I$. The high-order power matrix $\A^k$ indicates the path reachability within $k$ hops, which reflects connection relations of long-distant scope.
\end{itemize}
In the above formula, we successively conduct three operations on an input feature tensor: spatial aggregation, temporal aggregation, and feature transformation, finally produce the encoded feature tensor $\widetilde{\mcX}\in\mbR^{N\times D'\times T}$. It is worth noting that, the sorting on three operations could be randomly arranged, and cannot yet change the final calculation result because of the satisfactory property of the commutative law of tensor multiplication.

The aim is to estimate traffic speeds/flows of the next $T'$ moments at different locations. To this end, we define a transformation from encoded features to expected results as
\begin{align}
\hbY=\sigma(\widetilde{\X}_{(1)}\times \W + \b),
\end{align}
where the matrix $\W\in\mbR^{D'T\times T'}$ is the projection operation to be learned, $\b$ is a bias term, $\sigma$ is a non-linear activation function, and $\widetilde{\X}_{(1)}\in\mbR^{N\times D'T}$ takes a widely-used tensor notation to flatten other dimensions except the specified dimension (here the first dimension). The estimated result
$\hbY=[\hby^{t+1},\hby^{t+2},\cdots,\hby^{t+T'}]\in\mbR^{N\times T'}$ are expected to be equal to real values during training. Therefore, we have the objective function as
\begin{align}
\min_{\Theta,\W,\b}\quad \|\Y-\hbY\|_F^2 + \lambda \times \varsigma(\Theta,\W,\b),
\end{align}
where $\varsigma$ is the regularization term with $L_2$-norm on the parameters therein, and the balance parameter $\lambda$ is usually set as a small fraction (e.g., 1.0E-5).

\section{Factorized Tensor Convolution}

To reduce the computational burden, we take the tensor-factorized way to perform spatial-temporal graph convolution. The advantages are three folds: i) reduce computational complexity, ii) save memory resources, iii) full parallelizability in three operations (i.e., spatial aggregation, temporal aggregation, and feature transformation), and iv) reduce the effects of input noises. Please see a more detailed analysis, which is discussed in the following section.

To simplify the derivation, we consider the case when $k_\text{S}=1$ and $k_\text{T}=1$ in (\ref{eqn:wt_mcX}), and omit those unnecessary superscripts and subscripts. Then we can have the tensor-matrix multiplication as,
\begin{align}
\widetilde{\mcX}=\mcX\times_1\A_{\text{S}}~\widetilde{\times}_3~{\mcA_{\text{T}}}\times_2{\Theta},
\label{eqn:tbX}
\end{align}
where $\mcX\in\mbR^{N\times D\times T}$ is the input feature tensor, $\A_{\text{S}}\in\mbR^{N\times N}$ is the adjacency matrix of nodes, $\mcA_{\text{T}}\in\mbR^{N\times T\times T}$ is the temporal relation tensor of all nodes and $\Theta\in\mbR^{D'\times D}$ is the feature transformation matrix.

According to Tucker decomposition, we decompose the input tensor $\mcX$ into several components, i.e.,
\begin{align}
\mcX\approx \mcC\times_1 \X_\text{S}\times_2\X_\text{F}\times_3\X_\text{T},\label{eqn:mcX}
\end{align}
where the core tensor is $\mcC\in {\mbR^{{n}\times{d}\times{t}}}$, the spatial component is matrix $ \X_S\in{\mbR^{N\times{n}}}$, the feature component is matrix $\X_F\in{\mbR^{D\times{d}}}$, and the temporal component is matrix $ \X_T\in{\mbR^{T\times{t}}}$. The three components are independent of each other. Moreover, we often have the constraints, $n\ll N, t\ll T, d\ll D$, which largely reduce the computational burden on time and memory requirements. In the case of tensor decomposition, we can derive the following proposition about the computation of encoded features.

\begin{prop} Given the tensor decomposition of $\mcX$ in (\ref{eqn:mcX}), we can compute the spatial-temporal encoded feature $\widetilde{\mcX}$ in (\ref{eqn:tbX}) by the following formulas
	\begin{align}
	[\widetilde{\mcX}]_{k::}~~&\approx \mcC\times_1 [\tbX_\text{S}]_{k:}~{\times}_3~[{\mcX}_{\text{T}}]_{k::}\times_2\tbX_\text{F},\label{eqn:prop_mcX_k}\\
	\tbX_\text{S} &= \A_\text{S}\times \X_\text{S},\label{eqn:prop_tbX_S}\\
	[\widetilde{\mcX}_\text{T}]_{k::} &= [\mcA_\text{T}]_{k::}\times \X_\text{T},\label{eqn:widetilde_mcX}\\
	\tbX_\text{F} &= \Theta \times \X_\text{F},\label{eqn:prop_tbX_F}
	\end{align}
	where $\tbX_\text{S}\in \mbR^{{N}\times{n}}, \tbX_\text{F}\in \mbR^{{D'}\times{d}}$, and $\widetilde{\mcX}_\text{T}\in \mbR^{N\times{T}\times{t}}$. $[\tbX_\text{S}]_{k:}$ and $[{\mcX}_{\text{T}}]_{k::}$ correspond to spatial and temporal correlation information for node $k$, respectively. Moreover, the three components $\tbX_\text{S}, \tbX_\text{F}, \widetilde{\mcX}_\text{T}$ are irrelevant to each other and so may be calculated in a parallel way.
\end{prop}

\begin{proof}
	\renewcommand{\qedsymbol}{}
	We substitute (\ref{eqn:mcX}) into (\ref{eqn:tbX}), and then have
	\begin{align}
	\widetilde{\mcX}&= \mcX\times_1\A_{\text{S}}~\widetilde{\times}_3~{\mcA_{\text{T}}}\times_2{\Theta}\nonumber\\
	& \approx (\mcC\times_1 \X_\text{S}\times_2\X_\text{F}\times_3\X_\text{T}) \times_1\A_{\text{S}}~\widetilde{\times}_3~{\mcA_{\text{T}}}\times_2{\Theta}
	\end{align}
	According to the associative law and the commutative law for tensor-matrix multiplication, we can rewrite $\widetilde{\mcX}$ as
	\begin{align}
	\widetilde{\mcX}
	& \approx (\mcC\times_1 (\A_{\text{S}}\times\X_\text{S})\times_2(\Theta\times\X_\text{F})\times_3\X_\text{T}) ~\widetilde{\times}_3~{\mcA_{\text{T}}}\nonumber\\
	& =(\mcC\times_1 \tbX_\text{S}\times_2\tbX_\text{F}\times_3\X_\text{T}) ~\widetilde{\times}_3~{\mcA_{\text{T}}},
	\end{align}
	where $\tbX_\text{S}$ and $\tbX_\text{F}$ are defined in the proposition.
	Based on the batched tensor multiplication in (\ref{eqn:batch_cal}), we take the $k$-th slice of the first dimension w.r.t nodes, and rewrite it as
	\begin{align}
	[\widetilde{\mcX}]_{k::}
	& \approx [\mcC\times_1 \tbX_\text{S}\times_2\tbX_\text{F}\times_3\X_\text{T}]_{k::} ~{\times}_3~{[\mcA_{\text{T}}]_{k::}}\nonumber\\
	& =(\mcC\times_1 [\tbX_\text{S}]_{k:}\times_2\tbX_\text{F}\times_3[\X_\text{T}]) ~{\times}_3~{[\mcA_{\text{T}}]_{k::}}\nonumber\\
	& =\mcC\times_1 [\tbX_\text{S}]_{k:}\times_2\tbX_\text{F}\times_3({[\mcA_{\text{T}}]_{k::}}\times[\X_\text{T}])\nonumber\\
	& = \mcC\times_1 [\tbX_\text{S}]_{k:}\times_2\tbX_\text{F}\times_3[\widetilde{\mcX}_{\text{T}}]_{k::},
	\end{align}
	where $\widetilde{\mcX}_{\text{T}}$ is defined in (\ref{eqn:widetilde_mcX}). In the above derivation process, we sometimes preserve the original dimension when taking the slice operation, e.g., $[\tbX_\text{S}]_{k:}\in\mbR^{1\times n}$, which is clear in the context.
\end{proof}

Further, we can extend the above tensor calculation to the high-order cases, where $k_\text{S}$ and $k_\text{T}>1$. At this time, we only need to update the corresponding adjacency matrices with different orders. Finally, we can aggregate multi-scale receptive field responses as formulated in (\ref{eqn:wt_mcX}).

\textbf{Advantages of Factorization}~~In contrast to the original spatial-temporal tensor graph convolution,
the case of tensor decomposition have four aspects of advantages. Below we provide a detailed analysis.
\begin{itemize}
	\item[-] Computation complexity: As the crucial step is the computation of feature encoding in (\ref{eqn:tbX}), we mainly analyze its computational complexity. In the general convolution of tensor multiplication, the complexity is about $\mcO(N^2DT+ NDT^2+NDD'T)$.
	
	In the factorized version, we need to compute the formulas (\ref{eqn:mcX})$\sim$(\ref{eqn:prop_tbX_F}). The cost for Tucker decomposition with higher order
	orthogonal iteration~(HOOI)~\cite{de2000best} in (\ref{eqn:mcX}) is about $\mcO(NT(7d^2+d)+ND(7t^2+t)+TD(7n^2+n)+NTD(n+t+d)+N(n^2+n)+T(t^2+t)+D(d^2+d)+\frac{11}{3}(n^3+t^3+d^3))$ according to the analysis in the literature~\cite{comon2009tensor,elden2009newton}\footnote{Of course, some fast algorithms may also be employed to speed up tensor decomposition.}, and four last terms are often not taken into consideration. The complexity of calculations from (\ref{eqn:prop_tbX_S}) to (\ref{eqn:prop_tbX_F}) is about $\mcO(N^2n+DD'd+NT^2t)$, which contains the computation in spatial domain, temporal domain and feature domain. The decoding process in (\ref{eqn:prop_mcX_k}) is about $\mcO(Nndt+ NTdt+ND'Td)$.
	
	As there exist $n\ll N, d\ll D, t\ll T$ and $D'\sim D$ in practice, and $T\leq N$, $D\leq N$ or $D\sim N$ at most for large traffic networks, so the cost mainly takes at the step of tensor decomposition in (\ref{eqn:mcX}). Meantime, if we let $n<\sqrt{N}$, which is mostly feasible,
	the computational complexity of the factorized version is about $\mcO(NDT(n+d+t))$, which is far lower than the original $\mcO(N^2DT+ NDT^2+NDD'T)$.
	\item[-] Memory requirement: After tensor decomposition, we only need the memory space of $nN+tT+dD+ndt$, which is far less than the original memory space $NDT$. Subsequently, for the factorized version, the memory spaces in the next calculations would be decreased compared with the original version.
	\item[-] Parallel computation: After tensor decomposition, the computation in spatial domain, temporal domain, and feature domain, i.e., (\ref{eqn:prop_tbX_S})$\sim$ (\ref{eqn:prop_tbX_F}), can be conducted in a fully parallel way. In contrast, these three operations in (\ref{eqn:wt_mcX}) cannot be performed for parallelization.
	
	\item[-] Noise suppression: In tensor decomposition, we preserve those most crucial components by discarding those trivial components. Commonly, those trivial components contain more noise information. This phenomenon is validated for many subspace methods such as principal component analysis (PCA). Hence, our method can benefit from tensor decomposition, which can suppress noises to some extent. In the following experiment, we also verify this point by comparing the factorized version with the original ST-TGCN.
\end{itemize}
\begin{table*}[ht]
	\caption{Comparisons with state-of-the-art methods~(* means the value is pretty small).}
	\label{tab:table1}
	\centering
	\begin{tabular}{c|c|*4{c} *1{c|}*5{c}}
		\toprule
		\multirow{2}*{Model}&\multirow{2}*{Time}&\multicolumn{5}{c|}{SZ-taxi}&\multicolumn{5}{c}{Los-loop}\\
		{}&{}&$RMSE$&$MAE$&$Accuracy$&${R^2}$ &$var$&$RMSE$&$MAE$&$Accuracy$&${R^2}$ &$var$\\
		\midrule
		
		{HA}&15min      &4.2951  &2.7815  &0.7008  &0.8307 &0.8307  &7.4427  &4.0145  &0.8733  &0.7121  &0.7121\\
		\midrule
		
		\multirow{4}*{ARIMA}
		&15min   &7.2406  &4.9824  &0.4463  &*      &0.0035  &10.0439 &7.6832  &0.8275  &0.0025  &*\\
		&30min   &6.7899  &4.6765  &0.3845  &*      &0.0081  &9.3450  &0.6891  &0.8275  &0.0031  &* \\
		&45min   &6.7852  &4.6734  &0.3847  &*      &0.0087  &10.0508 &7.6924  &0.8273  &*       &0.0035\\
		&60min   &6.7708  &4.6655  &0.3851  &*      &0.0111  &10.0538 &7.6952  &0.8273  &*       &0.0036\\
		\midrule
		\multirow{4}*{SVR}
		&15min   &4.1455  &2.6233  &0.7112  &0.8423  &0.8424   &6.0084  &3.7285  &0.8977  &0.8123 &0.8146   \\
		&30min   &4.1628  &2.6875  &0.7100  &0.8410  &0.8413   &6.9588  &3.7248  &0.8815  &0.7492  &0.7523\\
		&45min   &4.1885  &2.7359  &0.7082  &0.8391  &0.8397   &7.7504  &4.1288  &0.8680  &0.6899  &0.6947  \\
		&60min   &4.2156  &2.7751  &0.7063  &0.8370  &0.8379   &8.4388  &4.5036  &0.8562  &0.6336  &0.5593  \\
		\midrule
		\multirow{4}*{GCN}
		&15min   &5.6596  &4.2367  &0.6107  &0.6654   &0.6655    &7.7922  &5.3525  &0.8673  &0.6843  &0.6844\\
		&30min      &5.6918  &4.2647  &0.6085  &0.6616   &0.6617    &8.3353  &5.6118  &0.8581  &0.6402  &0.6404\\
		&45min &5.7142  &4.2844  &0.6069  &0.6589   &0.6564    &8.8036  &5.9534  &0.8500  &0.5999  &0.6001\\
		&60min   &5.7361  &4.3034  &0.6054  &0.6590   &0.6564    &9.2657  &6.2892  &0.8421  &0.5583  &0.5593\\
		\midrule
		\multirow{4}*{GRU}
		&15min   &3.9994  &2.5955  &0.7249  &0.8329   &0.8329   &5.2182  &3.0602  &0.9109  &0.8576&0.8577\\
		&30min      &4.0942  &2.6906  &0.7184  &0.8249   &0.8250   &6.2802  &3.6505  &0.8931  &4.0915&0.7958\\
		&45min &4.1534  &2.7743  &0.7143  &0.8198   &0.8199   &7.0343  &4.5186  &0.8801  &0.7446&0.7451\\
		&60min    &4.0747  &2.7712  &0.7197  &0.8266   &0.8267   &7.6621  &0.7957  &0.8694  &0.6980&0.6984\\
		\midrule
		\multirow{4}*{STGCN}
		&15min   &3.9941  &2.6608  &0.7211  &0.8525  &0.8529    &6.0844  &3.3577  &0.8958  &0.8155  &0.8162\\
		&30min      &4.0336  &2.6937  &0.7186  &0.8498  &0.8503    &7.6831  &4.1249  &0.8686  &0.7066  &0.7098\\
		&45min &4.0553  &2.7205  &0.7173  &0.8482  &0.8491    &8.6429  &4.6632  &0.8523  &0.6295  &0.6358\\
		&60min   &4.0666  &2.7334  &0.7169  &0.8475  &0.8484    &9.4822  &5.1523  &0.8380  &0.5547  &0.5649\\
		\midrule
		\multirow{4}*{T-GCN}
		&15min  &3.9162  &2.7061  &0.7306  &0.8541   &0.8626   &5.1264  &3.1802  &0.9172  &0.8634  &0.8634\\
		&30min      &3.9617  &2.7452  &0.7275  &0.8523   &0.8523   &6.0598  &3.7466  &0.8968  &0.8098  &0.8100\\
		&45min &3.9950  &2.7666  &0.7552  &0.8509   &0.8509   &6.7065  &4.1158  &0.8857  &0.7679  &0.7684\\
		&60min    &4.0141  &2.7889  &0.7238  &0.8503   &0.8504   &7.2677  &4.6021  &0.8762  &0.7283  &0.7290\\
		\midrule
		\multirow{4}*{A3T-GCN}
		&15min   &3.8989  &2.6840  &0.7318  &0.8512  &0.8512    &5.0904  &3.1365  &0.9133  &0.8653  &0.8653\\
		&30min      &3.9228  &2.7038  &0.7302  &0.8493  &0.8493    &5.9974  &3.6610  &0.8979  &0.8137  &0.8137\\
		&45min &3.9461  &2.7261  &0.7286  &0.8474  &0.8474    &6.6840  &4.1712  &0.8861  &0.7694  &0.7705\\
		&60min    &3.9707  &2.7391  &0.7269  &0.8454  &0.8454    &7.0990  &4.2343  &0.8790  &0.7407  &0.7415\\
		\midrule
		\multirow{4}*{\tabincell{c}{factorized\\ST-TGCN}}
		&15min      &\textbf{3.1080}  &\textbf{2.0198}  &\textbf{0.7835}  &\textbf{0.9114}  &\textbf{0.9114}  &\textbf{3.5969}  &\textbf{2.2265}  &\textbf{0.9387}  &\textbf{0.9330}  &\textbf{0.9330}\\
		&30min      &\textbf{3.5181}  &\textbf{2.2951}  &\textbf{0.7548}  &\textbf{0.8865}  &\textbf{0.8865}  &\textbf{4.9283}  &\textbf{2.8690}  &\textbf{0.9159}  &\textbf{0.8749}  &\textbf{0.8749}\\
		&45min      &\textbf{3.6039}  &\textbf{2.3689}	&\textbf{0.7488}  &\textbf{0.8809} 	&\textbf{0.8809}  &\textbf{5.5733}  &\textbf{3.3600}  &\textbf{0.9049}  &\textbf{0.8402}  &\textbf{0.8409}\\
		&60min      &\textbf{3.7358}  &\textbf{2.4476}  &\textbf{0.7396}  &\textbf{0.8720} 	&\textbf{0.8720}  &\textbf{5.8225}  &\textbf{3.4772}  &\textbf{0.9006}  &\textbf{0.8258}  &\textbf{0.8263}\\
		\bottomrule
	\end{tabular}
	\vspace{-5pt}
\end{table*}
\section{Experiments}
In this section, we first introduce the two public datasets used to evaluate the methods, then compare the proposed factorized ST-TGCN with existing state-of-the-art works, and finally analyze our factorized ST-TGCN by conducting the ablation study.    

\subsection{Datasets} 
\textbf{SZ-taxi}: This dataset is collected from a taxi trajectory of Shenzhen, China. It defines 156 major roads as nodes, which are characterized by the speed at which taxis pass through the roads. And the adjacency matrix represents the positional connections between roads. These data are collected every 15 minutes, ranging from 1/1/2015 to 1/31/2015.

\textbf{Los-loop}: This dataset collected by loop detectors deployed on the highway of Los Angeles in America. It contains 207 nodes characterized by the speed collected by loop detectors. And adjacency matrix represents the positional connections between loop detectors. These data are collected every 5 minutes, ranging from 3/1/2012 to 3/7/2012.

For the experiments on these two datasets, we strictly follow the widely adopted protocol in previous works~\cite{zhao2019t}. Given the SZ-taxi or Los-loop dataset, we split 80\% of data for training and 20\% of data for testing. We predict the traffic speed for prediction horizons from 15 to 60 minutes by using previous data of 12 time slices.

\subsection{Implementation details}
All experiments in this works are conducted on a single NVIDIA RTX2080Ti. In the experiments, We firstly utilize two fully connected layers to map nodes into a 128-dimension feature vectors. Then we construct two-layer graph convolutional neural networks with 128 and 64 filters respectively. In GCN layers, we use $relu(\cdot)$ as an activation function and order number $p$ is set equal to 2. Finally, we train our proposed factorized ST-TGCN for 500 epochs with a learning rate of 0.001 and the batch size as 32.

\subsection{Metrics}
Following~\cite{zhao2019t,zhu2020a3t}, we adopt the five metrics to comprehensively evaluate factorized ST-TGCN, 

\begin{itemize}
	\item[-] Root Mean Squared Error 
	\begin{align}\small
	RMSE=\sqrt{\frac{1}{N}\sum_{i=1}^{N}(\hat{y}^{i}-y^i)^2},
	\end{align}
	\item[-]  Mean Absolute Error
	\begin{align}\small
	MAE=\frac{1}{N}\sum_{i=1}^{N}|\hat{y}^i-y^i|,
	\end{align}
	\item[-] Accuracy
	\begin{align}\small
	Accuracy=1-	\frac{\sqrt{\sum_{i=1}^{N}{|\hat{y}^i-y^i|^2}}}{\sqrt{\sum_{i=1}^{N}{|y^i|^2}}},
	\end{align}
	\item[-] R-squared
	\begin{align}\small
	R^2=1-\frac{\sum_{i=1}^{N}{|\hat{y}^i-y^i|^2}}{\sum_{i=1}^{N}{|y^i-\frac{1}{N}\sum_{k=1}^{N}y^k|^2}},
	\end{align}
	\item [-] Variance 
	\begin{align}\small
	var=1-\frac{\frac{1}{N}\sum_{i=1}^{N}|(y^i-\hat{y}^i)-\frac{1}{N}\sum_{k=1}^{N}(y^k-\hat{y}^k)|^2}{\frac{1}{N}\sum_{i=1}^{N}|y^i-\frac{1}{N}\sum_{k=1}^{N}y^k|^2},
	\end{align}
\end{itemize}
where $y^i$, $\hat{y}^i$ represent ground-truth value and prediction value, $N$ is the number of nodes. As the square root operation of RMSE amplifies the gaps between ground-truth and prediction, MAE metric is introduced. Considering RMSE and MAE merely calculate the distance between the real data and the predicted data, three additional metrics named Accuracy, $R^2$, $var$ are introduced, where $R^2$ is the coefficient of determination and $var$ explains the variance score.

\begin{figure*}[h]
	\centering
	\vspace{-0.5cm}
	\begin{minipage}[t]{1\textwidth}
		\centering
		\includegraphics[width=17 cm]{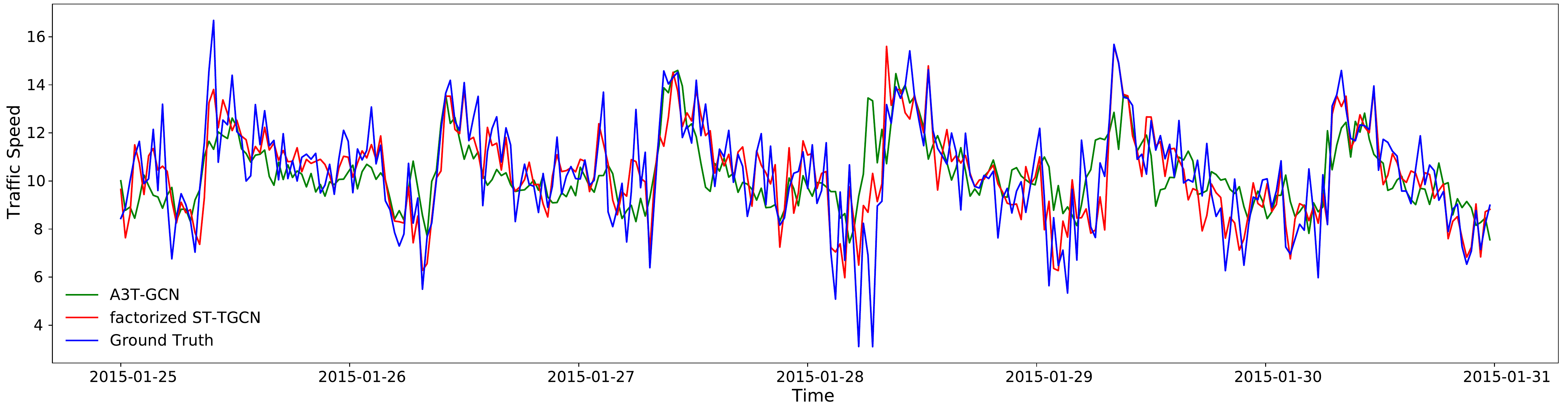}
		\centering
	\end{minipage}
	\caption{Coarse results of the speed prediction on the SZ-taxi for prediction horizons of 15 minutes}
	\label{fig:fig1}
\end{figure*}
\begin{figure*}[h]
	\centering
	\begin{minipage}[t]{1\textwidth}
		\centering
		\includegraphics[width=17cm]{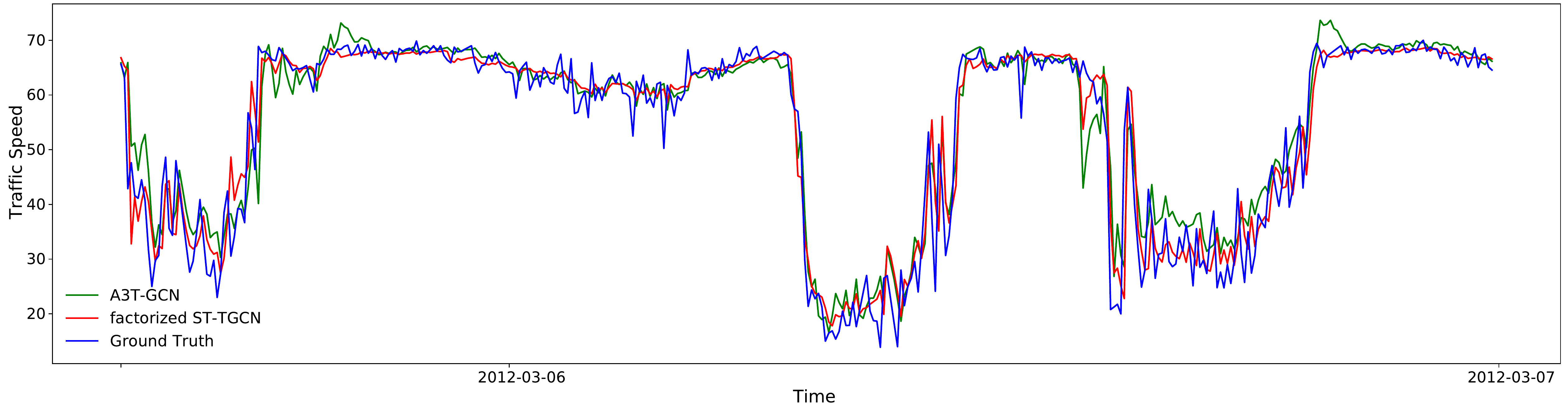}
		\centering
	\end{minipage}
	\caption{Coarse results of the speed prediction on the Los-loop for prediction horizons of 15 minutes}
	\label{fig:fig2}
	
\end{figure*}
\begin{figure*}[h]
	\centering
	\begin{minipage}[t]{0.49\textwidth}
		\centering
		\includegraphics[width=8cm]{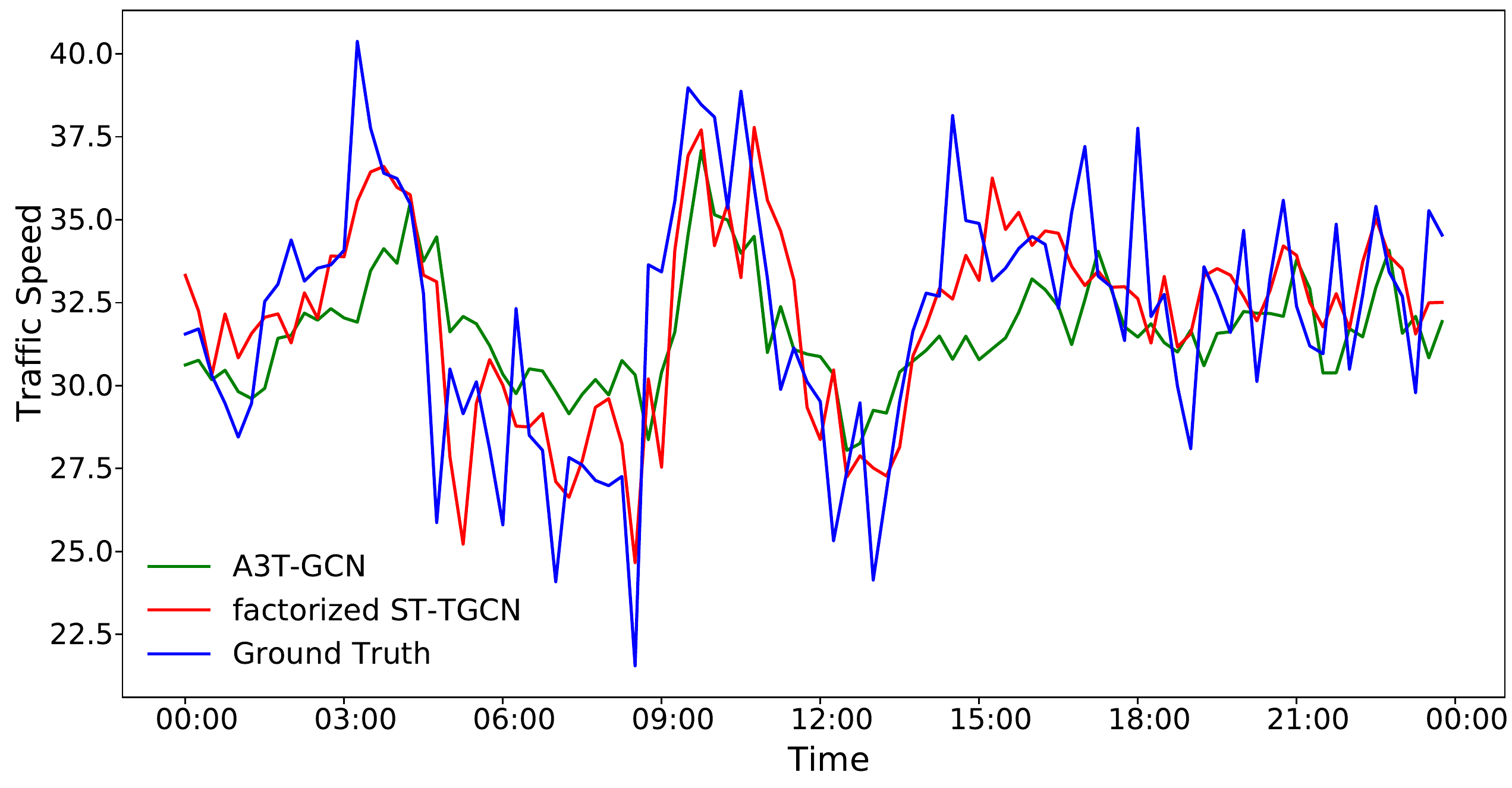}
		\centering
	\end{minipage}
	\begin{minipage}[t]{0.49\textwidth}
		\centering
		\includegraphics[width=8cm]{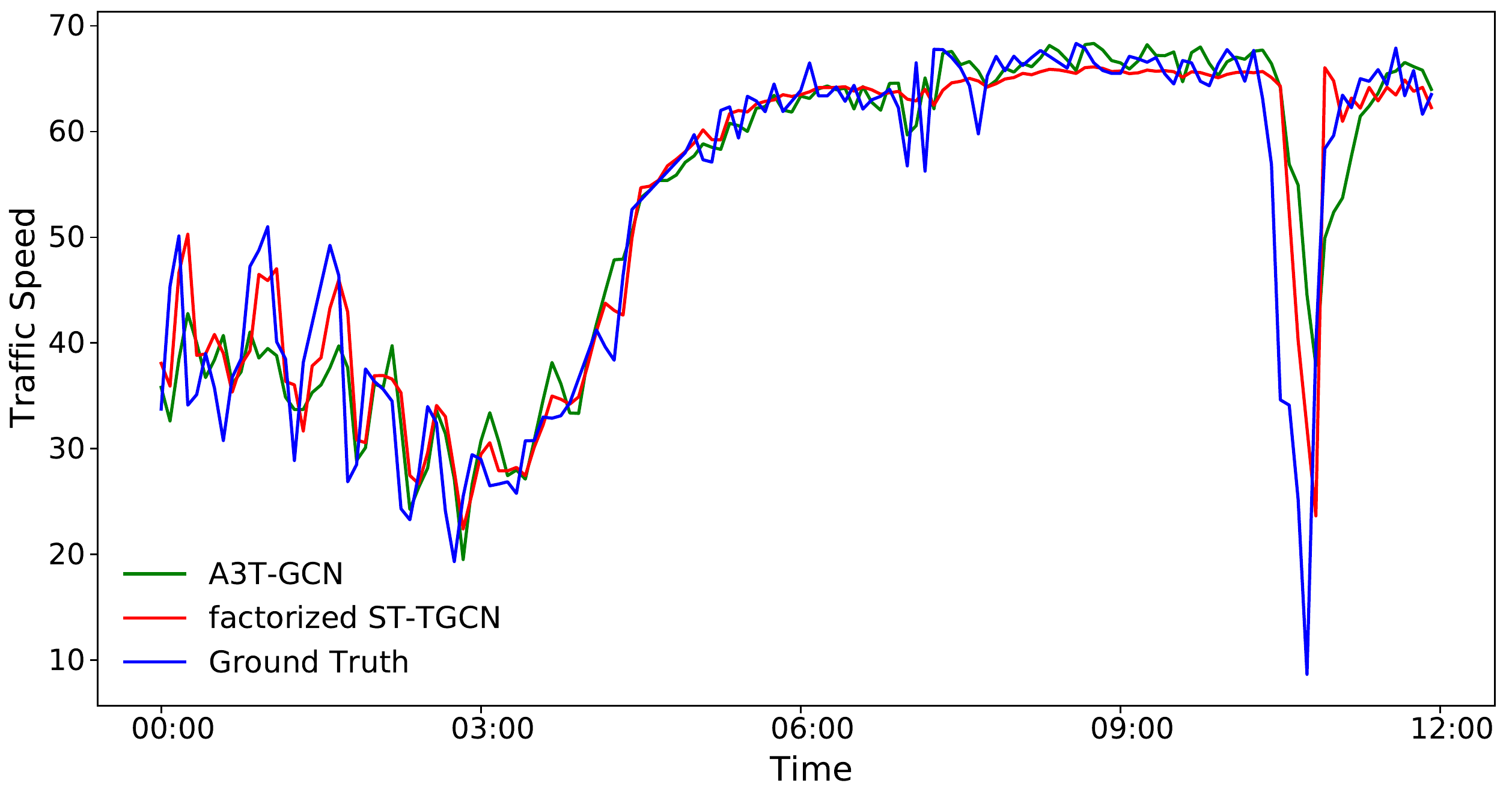}
		\centering
	\end{minipage}
	
	\caption{Fine results of the one day speed prediction on the SZ-taxi (left) and Los-loop (right) for prediction horizons of 15 minutes}
	\label{fig:fig3}
	\vspace{-10pt}
\end{figure*}
\subsection{Comparisons with the state-of-the-arts}
To comprehensively evaluate our proposed factorized ST-TGCN, we compare it with traditional methods~(HA, ARIMA~\cite{ahmed1979analysis}, and SVR~\cite{smola2004tutorial}) and deep learning-based methods~(GCN~\cite{kipf2016semi},  GRU~\cite{zhao2019t,cho2014properties}, STGCN~\cite{yu2017spatio}, T-GCN~\cite{zhao2019t}, and A3T-GCN~\cite{zhu2020a3t}). Since HA model simply computes an average of past traffic speed and uses the identical value as prediction value for horizons from 15 minutes to 60 minutes, we report its performance for the prediction horizons of 15 minutes. According to Table \ref{tab:table1}, overall, deep learning-based methods generally outperform the traditional ones except for the GCN which doesn't consider the temporal dynamics in the traffic data.

Compared with the SVR model for the prediction horizons of 15 minutes, factorized ST-TGCN obtains 7\% and 4\% improvement in $R^2$ on the datasets SZ-taxi and Los-loop respectively. Such improvement can also be verified on RMSE, MAE, Accuracy, and $var$ metrics. Although SVR is capable of achieving satisfactory results among traditional models, it is not applicable to large datasets. Moreover, these traditional methods merely model temporal correlations among traffic data, while factorized ST-TGCN takes irregular spatial correlations and non-linear temporal correlations into consideration. Then we compare factorized ST-TGCN with the deep learning-based models. Compared with GCN, our model increases about 25\% in $R^2$ on the SZ-taxi and Los-loop. GCN does not show great performance for traffic prediction, which is caused by GCN absolutely ignores the temporal dependence. Compared with GRU that only takes temporal correlations into account, factorized ST-TGCN increases $R^2$ by about 8\% on the SZ-taxi and Los-loop for the prediction horizons of 15 minutes. These suggest that take one type of the dependence into account is far from sufficient. Finally, we compare factorized ST-TGCN with deep learning-based models that take spatial-temporal dependence into consideration. A3T-GCN utilizes GCN to model the spatial topology structured of traffic data and uses GRU to learns the short-time trend in time series. Moreover, it introduces attention mechanism to re-weight the influence of historical traffic states. Compared with A3T-GCN which obtains the best results so far, factorized ST-TGCN decreases the RMSE and MAE about 20\% and 25\% for the prediction horizons of 15 minutes on the SZ-taxi. Besides, our model increases Accuracy, $R^2$, and $var$ about 5\%, 6\%, and 6\% respectively. Such improvement can also be validated on the Los-loop.


We also explore the capability of long-term prediction for our proposed factorized ST-TGCN. We compare factorized ST-TGCN to other models when prediction horizons vary from 15 to 60 minutes. We conclude that factorized ST-TGCN is capable of maintaining relatively stable performance, especially for the SZ-taxi. Moreover, we visualize the results of traffic prediction in Fig.~\ref{fig:fig1}, Fig.~\ref{fig:fig2}~(for the coarse time granularity), and Fig.~\ref{fig:fig3}~(for the fine time granularity). It can be observed that factorized ST-TGCN shows different performance on the two datasets, we infer that this is caused by speed changes in the SZ-taxi is more dramatic than those in the Los-loop.



\subsection{Ablation study}
\textbf{Component analysis of factorized ST-TGCN:}~To confirm that both spatial dependence and temporal dependence are crucial to traffic speed prediction, we devise a spatial tensor graph convolutional neural network~(factorized S-TGCN) and a temporal tensor graph convolutional neural network~(factorized T-TGCN). Factorized S-TGCN merely regards nodes that are joint in the spatial domain as neighbor nodes, factorized T-TGCN considers nodes are just connected in the temporal pattern and our proposed factorized ST-TGCN integrates them. Table \ref{tab:table2} shows the performance of three models on the two datasets. Compared with factorized S-TGCN,  with approximately 4\% improvement in $R^2$ is attained by factorized ST-TGCN on the SZ-taxi. In addition to this, factorized ST-TGCN obtains about 1\% improvement compared with factorized T-TGCN. These demonstrate that spatial-temporal dependence is important. We consider that the better performance of factorized ST-TGCN not only attributes to comprehensively consider spatial-temporal dependence but also due to tensor decomposition operation. Therefore, we further explore the Tucker decomposition operation in the following parts. 
\begin{table}[!h]\footnotesize
	\renewcommand\arraystretch{0.8}
	\renewcommand\tabcolsep{3.0pt}
	\caption{Component analysis of factorized ST-TGCN}
	\label{tab:table2}
	\centering
	\begin{tabular}{cc*5{c}}
		\toprule
		Dataset&Model &{$RMSE$}&{$MAE$}&{$Accuracy$}&{${R^2}$}	&{$var$} \\\midrule
		\multirow{5}*{SZ-taxi}&{\tabincell{c}{factorized\\S-TGCN}}     &3.6849&2.4511	&0.7433		 &0.8754		&0.8754\\
		&{\tabincell{c}{factorized\\T-TGCN}}      &3.2894	&2.1537		 &0.7709		&0.9007&0.9007\\
		&{\tabincell{c}{factorized\\ST-TGCN}} &\textbf{3.1080}  &\textbf{2.0198}  &\textbf{0.7835}  &\textbf{0.9114}  &\textbf{0.9114} \\
		\midrule
		\multirow{5}*{Los-loop}&{\tabincell{c}{factorized\\S-TGCN}}     &4.2940	&2.5598		 &0.9268		&0.9045&0.9053\\
		&{\tabincell{c}{factorized\\T-TGCN}}     &4.5326	&2.7062		 &0.9227		&0.89367&0.8937\\
		&{\tabincell{c}{factorized\\ST-TGCN}}  &\textbf{3.5969}  &\textbf{2.2265}  &\textbf{0.9387}  &\textbf{0.9330}  &\textbf{0.9330}\\
		\bottomrule
	\end{tabular}
\end{table}

\textbf{The effectiveness of Tucker decomposition:}~To explore the effectiveness of Tucker decomposition, we conduct comparative experiments between ST-TGCN and factorized ST-TGCN. ST-TGCN directly conduct a multi-graph convolutional neural network in the spatial-temporal domains without tensor decomposition operation. As shown in Table \ref{tab:table3}, the factorized ST-TGCN is capable of showing better performance than the ST-TGCN model. For the prediction horizons of 15 minutes on the dataset SZ-taxi, factorized ST-TGCN obtains about 5\% improvement in $R^2$. Moreover, as shown in Table \ref{tab:table4}, factorized ST-TGCN is capable of achieving roughly equivalent improvement on the Los-loop dataset for the prediction horizons of 15 minutes. We also observe factorized ST-TGCN still outperforms ST-TGCN with prediction horizons from 15 to 60 minutes, so factorized ST-TGCN successfully maintains effectiveness and robustness for the long-term prediction. Such improvement demonstrates that Tucker decomposition as one of the higher-order generalizations of the matrix SVD and PCA is useful for compression and can remove irrelevant components to improve performance.
\begin{table}[!h]\footnotesize
	\renewcommand\tabcolsep{4.0pt}
	\caption{Tucker decomposition on the SZ-taxi}
	\label{tab:table3}
	\centering
	\begin{tabular}{c{c}*5{c}}
		\toprule  Model &Time&{$RMSE$}&$MAE$&$Accuracy$&${R^2}$ &$var$\\
		\midrule
		\multirow{4}*{ST-TGCN}
		&15min&3.8976&2.5799&0.7285&0.8606&0.8607\\
		&30min&3.9261&2.5927&0.7264&0.8586&0.8586\\
		&45min&3.9362&2.5966&0.7256&0.8579&0.8579\\
		&60min&3.9541&2.6128 &0.7244&0.8565&0.8565\\
		\midrule
		\multirow{4}*{\tabincell{c}{factorized\\ST-TGCN}}
		&15min      &\textbf{3.1080}  &\textbf{2.0198}  &\textbf{0.7835}  &\textbf{0.9114}  &\textbf{0.9114}  \\
		&30min      &\textbf{3.5181}  &\textbf{2.2951}  &\textbf{0.7548}  &\textbf{0.8865}  &\textbf{0.8865}\\  
		&45min      &\textbf{3.6039}  &\textbf{2.3689}	&\textbf{0.7488}  &\textbf{0.8809} 	&\textbf{0.8809} \\ 
		&60min      &\textbf{3.7358}  &\textbf{2.4476}  &\textbf{0.7396}  &\textbf{0.8720} 	&\textbf{0.8720}  \\
		\bottomrule
	\end{tabular}
\end{table}
	\vspace{-5pt}
\begin{table}[!h]\footnotesize
	\renewcommand\tabcolsep{4.0pt}
	\caption{Tucker decomposition on the Los-loop}
	\label{tab:table4}
	\centering
	\begin{tabular}{{c}{c}*5{c}}
		\toprule Model &Time&{$RMSE$}&$MAE$&$Accuracy$&${R^2}$ &$var$\\
		\midrule
		\multirow{5}*{ST-TGCN}
		&15min 		&4.5068&2.6846&0.9232&0.8948&0.8948\\
		&30min		&5.2875&3.0651&0.9098&0.8560&0.8561\\
		&45min 		&5.8455&3.4912&0.9002&0.8243&0.8247\\
		&60min		&6.2267&3.6849&0.8937&0.8007&0.8008\\
		\midrule
		\multirow{4}*{\tabincell{c}{factorized\\ST-TGCN}}
		&15min   &\textbf{3.5969}  &\textbf{2.2265}  &\textbf{0.9387}  &\textbf{0.9330}  &\textbf{0.9330}\\
		&30min   &\textbf{4.9283}  &\textbf{2.8690}  &\textbf{0.9159}  &\textbf{0.8749}  &\textbf{0.8749}\\
		&45min   &\textbf{5.5733}  &\textbf{3.3600}  &\textbf{0.9049}  &\textbf{0.8402}  &\textbf{0.8409}\\
		&60min   &\textbf{5.8225}  &\textbf{3.4772}  &\textbf{0.9006}  &\textbf{0.8258}  &\textbf{0.8263}\\
		\bottomrule
	\end{tabular}
	\vspace{-5pt}
\end{table}

\textbf{Perturbation analysis:}~The traffic data collected by sensors is inevitably noisy, the ability to resist noise perturbation reflects the effectiveness and robustness of the model. Here, we conduct a perturbation analysis to validate these properties of our proposed model. We perform perturbations to traffic speed data by adding random Gaussian noise~\cite{rasmussen2003gaussian}. The probability density function (PDF) of Gaussian distribution $N\in(0,\sigma^2)$ is formulated as
\begin{align}
p(x)=\frac{1}{\sqrt{2\pi}\sigma}exp(-\frac{(x-\mu)^2}{2\sigma^2}),
\end{align}
where $\mu$ is the mean value and $\sigma$ is the standard deviation. We compare factorized ST-TGCN with ST-TGCN to validate the ability to resist noise perturbation. We set $\sigma={\{0,0.2,0.4,1,2,4\}}$, where $\sigma=0$ means there is no Gaussian noise added to traffic data. 
\begin{figure}[h]
	\vspace{-0.5cm}
	\centering
	\begin{minipage}[t]{0.47\textwidth}
		\centering
		\includegraphics[width=8.5cm]{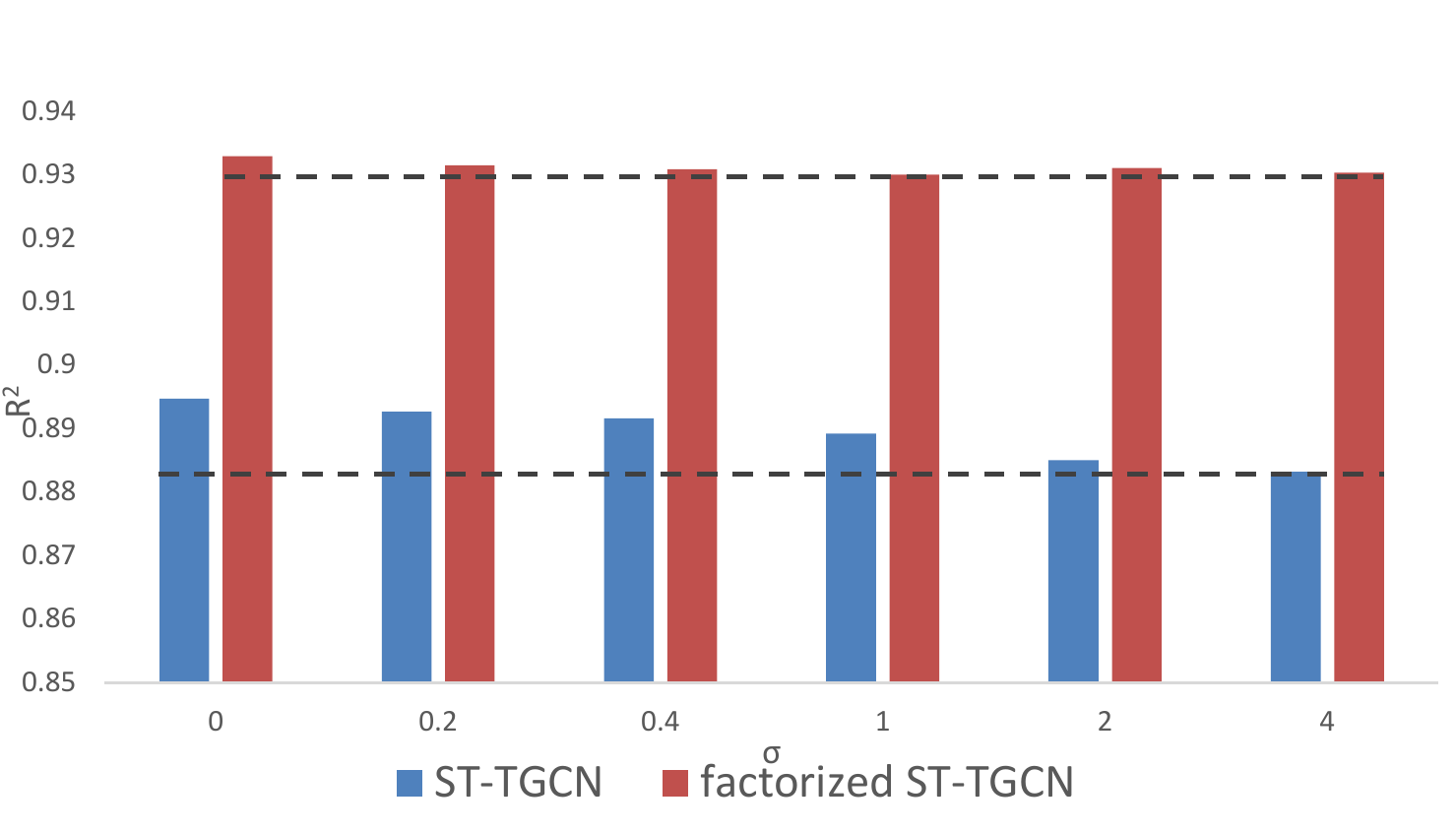}
		\centering
	\end{minipage}
	\caption{Perturbation analysis on the Los-loop for prediction horizons of 15 minutes}
	\label{fig:fig4}
	\vspace{-10pt}
\end{figure}
As shown in Fig.~\ref{fig:fig4}, on the dataset Los-loop, ST-TGCN falls much more than factorized ST-TGCN in the metric of $R^2$. Therefore, we conclude that i) ST-TGCN introduces the idea of multi-graph to deal with traffic prediction, so ST-TGCN is capable of resisting perturbation to some extent. However, ST-TGCN fails to maintain robustness when $\sigma$ increases. ii) Our factorized ST-TGCN outperforms ST-TGCN and shows great robustness with perturbation growing.

\textbf{The components number of factorized ST-TGCN:}~In the Tucker decomposition, choosing how many components to use for describing the original tensor is related to noise compression and computational complexity. Based on three criteria proposed by~\citet{kiers2001three}, for our tensor data~$\mcX\in\mbR^{N\times D\times T}$, we have i) $N$ denotes the number of nodes, the number of $N$ tends to be larger, which is Tucker decomposition focus on. ii) $T$ and $D$ are constant value that is not affected by the scale of datasets and  $T$ tends to be small. As Table \ref{tab:table5} shows, we compare the performance of factorized ST-TGCN with a varying number of components in mode $N$, where core tensor $\mcC\in {\mbR^{{n}\times{d}\times{t}}}$. Considering $d$ and $t$ can be omitted, we vary them with $n$. When $n$ equals $N^{\frac{1}{2}}$, $t$ equals $T^{\frac{1}{2}}$ and $d$ equals $D^{\frac{1}{2}}$. Specifically, we have component matrices  $\X_S\in\mbR^{N\times {n} }, \X_F\in\mbR^{D\times {d}}$, and $\X_T\in\mbR^{T\times {t}}$. In Table \ref{tab:table5}, factorized ST-TGCN shows best performance when $n$ is equals to $N^{\frac{1}{2}}$. Moreover, we compare computational time and memory space in Table \ref{tab:table6}, results suggest that an appropriate number of components is capable of improving performance as well as reducing computational time and memory space. 

\begin{table}[!h]
	\renewcommand\tabcolsep{4.0pt}
	\caption{Comparisons on the number of components}
	\label{tab:table5}
	\centering
	\begin{tabular}{cc*5{c}}
		\toprule
		Dataset&$n$ & {$RMSE$}&  {$MAE$} &{$Accuracy$}&{$R^2$} &{$var$}\\
		\midrule
		\multirow{3}*{SZ-taxi}
		&$n=N$	                &3.8980  			&2.5828    			&0.7285				&0.8606				&0.8606\\
		& $n=N^{\frac{1}{2}}$	&\textbf{3.1080}  	&\textbf{2.0198}  	&\textbf{0.7835}  	&\textbf{0.9114}  	&\textbf{0.9114} \\
		& $n=N^{\frac{1}{3}}$	&3.6997  			&2.4670  			&0.7423				&0.8744 			&0.8746 \\
		\midrule
		\multirow{3}*{Los-loop}
		& $n=N$	            	&4.9783  			&2.9430    			&0.9151				&0.8716				&0.8758\\
		& $n=N^{\frac{1}{2}}$	 &\textbf{3.5969}  &\textbf{2.2265}  &\textbf{0.9387}  &\textbf{0.9330}  &\textbf{0.9330}\\
		& $n=N^{\frac{1}{3}}$	&3.7119  	&2.2918  	&0.9367	&0.9286	&0.9287 \\
		\bottomrule
	\end{tabular}
		\vspace{-5pt}
\end{table}
\begin{table}[!h]
	\renewcommand\tabcolsep{2.5pt}
	\caption{Comparisons on the numbers of components}
	\label{tab:table6}
	\centering
	\begin{tabular}{cc*5{c}}
		\toprule
		Dataset&$n$ &{\tabincell{c}{Computational\\Time(s)}} &{\tabincell{c}{Memory\\Space}} \\
		\midrule
		\multirow{3}*{SZ-taxi}
		&	$n=N$				&31.49			&$NDT$\\
		&$n=N^{\frac{1}{2}}$ 	&0.35			&$N^\frac{3}{2}+T^\frac{3}{2}+D^\frac{3}{2}+(NTD)^\frac{1}{2}$\\
		&$n=N^{\frac{1}{3}}$	&\textbf{0.25}	&\bm{$N^\frac{4}{3}+T^\frac{4}{3}+D^\frac{4}{3}+(NTD)^\frac{1}{3}$}\\
		\midrule
		\multirow{3}*{Los-loop}
		&$n=N$					&43.80			&$NDT$\\
		&$n=N^{\frac{1}{2}}$	&0.65			&$N^\frac{3}{2}+T^\frac{3}{2}+D^\frac{3}{2}+(NTD)^\frac{1}{2}$\\
		&$n=N^{\frac{1}{3}}$	&\textbf{0.54}	&\bm{$N^\frac{4}{3}+T^\frac{4}{3}+D^\frac{4}{3}+(NTD)^\frac{1}{3}$}\\
		\bottomrule
	\end{tabular}
\end{table}
\section{Conclusion}
In this paper, we have proposed a factorized spatial-temporal tensor convolutional neural networks~(factorized ST-TGCN) for traffic prediction and introduced tensor decomposition for efficient tensor graph convolution. Experimental results on two real-world traffic datasets have shown factorized ST-TGCN outperforms other state-of-the-art methods.
Besides, factorized ST-TGCN shows its superiorities in computational complexity, memory requirement, parallel computational efficiency, and noise suppression. The benefits are more obvious as the dataset size increases. In the future, we plan to apply factorized ST-TGCN in other related tasks with spatial and temporal characteristics.
\ifCLASSOPTIONcaptionsoff
  \newpage
\fi

\footnotesize{
\bibliographystyle{plainnat}
\bibliography{IEEEabrv,mybib}{}
\bibliographystyle{IEEEtran}
}
\begin{IEEEbiography}[{\includegraphics[width=1in,height=1.25in,clip,keepaspectratio]{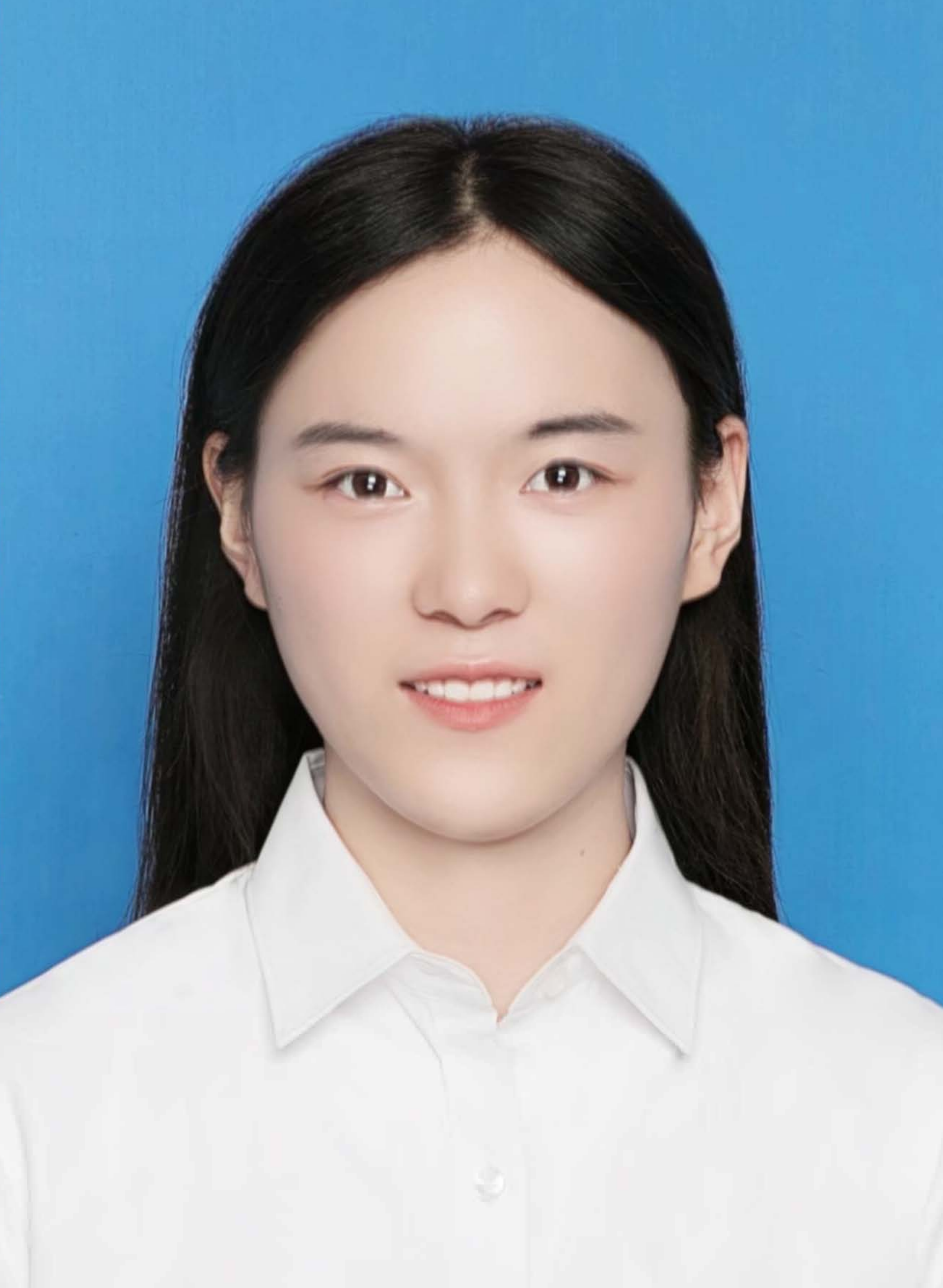}}]{Xuran Xu} received the B.S. degree in Software Engineering from 
	Jiangsu University of Science and Technology, Zhenjiang, China, in 2019. She is currently working toward the M.S. degree at the School of Computer Science and Engineering, Nanjing University of Technology and Science, Nanjing, 210094, China. Her current research interests include graphical model, tensor graph convolutional network, intelligent transportation systems, and deep learning.
\end{IEEEbiography}
\vspace{-1.2cm}
\begin{IEEEbiography}[{\includegraphics[width=1in,height=1.25in,clip,keepaspectratio]{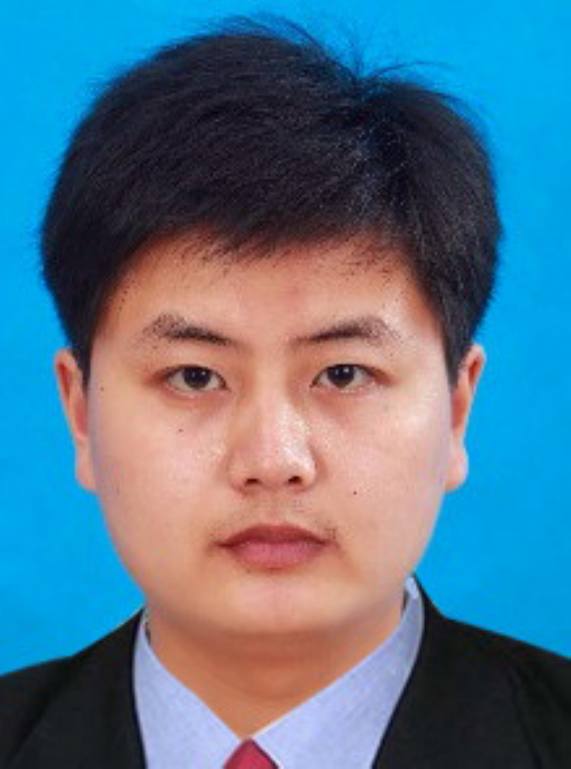}}]{Tong Zhang} received the B.S. degree in information science and technology from Southeast University, Nanjing, China, in 2011, the M.S. degree from the Research Center for Learning Science, Southeast University, in 2014, and the Ph.D. degree in information and communication engineering at Southeast University in 2018. Now he is a lecture in the School of Computer Science and Engineering from Nanjing University of Science and Technology, Nanjing, China.
His interests include pattern recognition, affective computing, and computer vision.
\end{IEEEbiography}
\vspace{-1.2cm}
\begin{IEEEbiography}[{\includegraphics[width=1in,height=1.25in,clip,keepaspectratio]{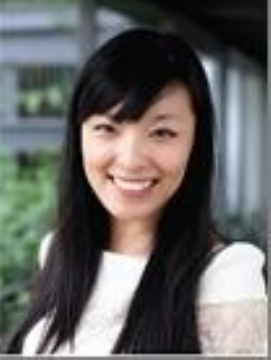}}]{Chunyan Xu} received the Ph.D. degree from the School of Computer Science and Technology, Huazhong University of Science and Technology, 2015. From 2013 to 2015, she was a visiting scholar in the Department of Electrical and Computer Engineering at National University of Singapore. Now she works in the School of Computer Science and Engineering from Nanjing University of Science and Technology, Nanjing, 210094, China. Her research interests include computer vision, manifold learning and deep learning.
\end{IEEEbiography}
\vspace{-1.2cm}
\begin{IEEEbiography}[{\includegraphics[width=1in,height=1.25in,clip,keepaspectratio]{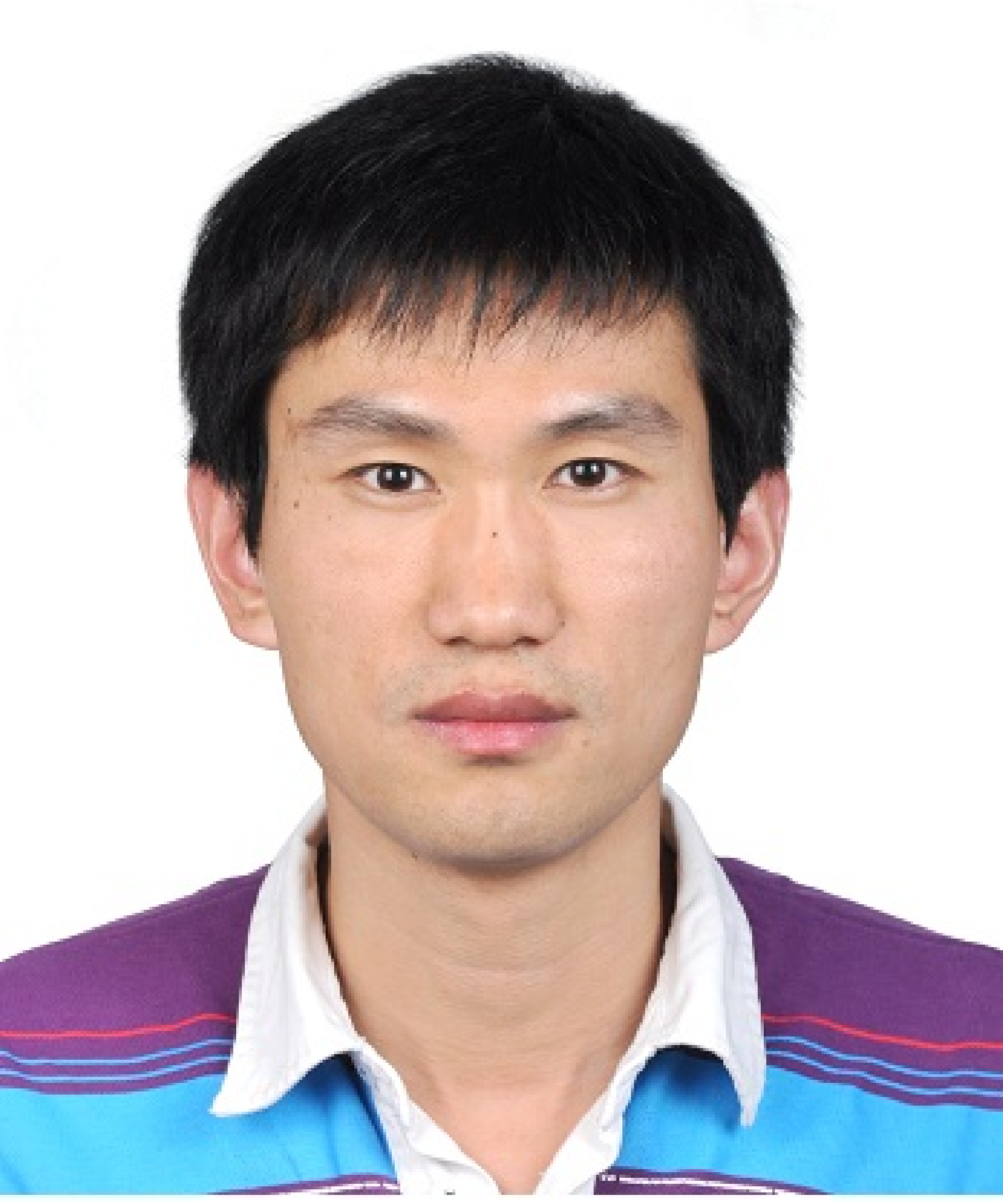}}]{Zhen Cui} received the B.S., M.S., and Ph.D. degrees from Shandong Normal University, Sun Yat-sen University, and Institute of Computing Technology (ICT), Chinese Academy of Sciences in 2004, 2006, and 2014, respectively. He was a Research Fellow in the Department of Electrical and Computer Engineering at National University of Singapore (NUS) from 2014 to 2015. He also spent half a year as a Research Assistant on Nanyang Technological University (NTU) from Jun. 2012 to Dec. 2012. Currently, he is a Professor of Nanjing University of Science and Technology, China. His research interests mainly include deep learning, computer vision and pattern recognition.
\end{IEEEbiography}
\vspace{-1.2cm}
\begin{IEEEbiography}[{\includegraphics[width=1in,height=1.25in,clip,keepaspectratio]{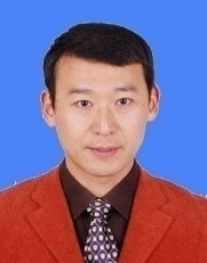}}]{Jian Yang} received the Ph.D. degree from Nanjing University of Science and Technology (NUST), on the subject of pattern recognition and intelligence systems in 2002. In 2003, he was a postdoctoral researcher at the University of Zaragoza. From 2004 to 2006, he was a Postdoctoral Fellow at Biometrics Centre of Hong Kong Polytechnic University. From 2006 to 2007, he was a Postdoctoral Fellow at Department of Computer Science of New Jersey Institute of Technology. Now, he is a Chang-Jiang professor in the School of Computer Science and Engineering of NUST. He is the author of more than 100 scientific papers in pattern recognition and computer vision. His journal papers have been cited more than 4000 times in the ISI Web of Science, and 9000 times in the Web of Scholar Google. His research interests include pattern recognition, computer vision and machine learning. Currently, he is/was an associate editor of Pattern Recognition Letters, IEEE Trans. Neural Networks and Learning Systems, and Neurocomputing. He is a Fellow of IAPR.
\end{IEEEbiography}

\end{document}